\documentclass[10pt]{article}
\usepackage{fancyhdr}
\usepackage{extramarks}
\usepackage{amsmath}
\usepackage{amsthm}
\usepackage[utf8]{inputenc}   
\usepackage[T1]{fontenc}
\usepackage{newunicodechar}   
 \usepackage{dblfloatfix} 
\usepackage{float}

\usepackage{booktabs}
\usepackage{tabularx}

\newunicodechar{α}{$\alpha$}
\newunicodechar{β}{$\beta$}
\newunicodechar{γ}{$\gamma$}
\newunicodechar{κ}{$\kappa$}
\newunicodechar{δ}{$\delta$}

\newunicodechar{–}{--}       
\newunicodechar{—}{---}      
\newunicodechar{’}{'}        
\newunicodechar{±}{$\pm$}

\newunicodechar{­}{}
\usepackage{amsfonts}
\usepackage{siunitx}
\usepackage{tikz}
\usepackage[plain]{algorithm}
\usepackage{algpseudocode}
\usepackage{multirow}
\usepackage{booktabs}
\usepackage{graphicx}
\usepackage{subfigure}
\usepackage[margin=1in]{geometry}
\usepackage{booktabs}
\usepackage{makecell}
\usepackage{array} 
\usepackage[colorlinks,linkcolor=black,anchorcolor=black,citecolor=black,urlcolor=blue]{hyperref}
\usepackage{hyphenat}
\usepackage{amsmath,bm}
\usepackage{booktabs}
\usepackage{mathtools}
\usepackage{amssymb}
\usepackage{tikz-cd}
\usepackage{caption}
\usepackage{capt-of}
\usepackage{mciteplus}
\usepackage{cite}
\usepackage{mathrsfs}
\usepackage[title,titletoc,toc]{appendix}
\usepackage{xr}
\usepackage{parskip}
\usepackage{soul}
\usepackage{textcomp}
\usepackage[colaction]{multicol}
\usepackage[switch]{lineno}
\usepackage{lipsum}
\usepackage{etoolbox}
\usepackage{longtable}
\usepackage{array}
\usepackage{tablefootnote}
\usepackage{ragged2e}
\usepackage{soul}
\newcolumntype{C}[1]{>{\centering\arraybackslash}p{#1}}
\usetikzlibrary{automata,positioning}
\usepackage{float}

\usetikzlibrary{automata,positioning}

\newtheorem{remark}{Remark}[section]

\newtheorem{prop}{Proposition}[section]

\usepackage{amsthm}  %
\theoremstyle{definition}
\newtheorem{definition}{Definition}[section]

\usepackage{xr}
\begin{document}
	\title{A Hierarchical Sheaf Spectral Embedding Framework for Single-Cell RNA-seq Analysis}
    
	\author{Xiang Xiang Wang$^{1}$,  and Guo-Wei Wei$^{1,2,3}$\footnote{
			Corresponding author.		Email: weig@msu.edu} \\
		\\
		$^1$ Department of Mathematics, \\
		Michigan State University, East Lansing, MI 48824, USA.\\
        $^2$ Department of Biochemistry and Molecular Biology,\\
		Michigan State University, East Lansing, MI 48824, USA.  \\
		$^3$ Department of Electrical and Computer Engineering,\\
		Michigan State University, East Lansing, MI 48824, USA. \\
        \\
	}
	\date{\today} 
	
	\maketitle
	
\begin{abstract}
Single-cell RNA-seq   data analysis typically requires representations that capture heterogeneous local structure across multiple scales while remaining stable and interpretable.
In this work, we propose a hierarchical sheaf spectral embedding (HSSE) framework that constructs informative cell-level features based on persistent sheaf Laplacian analysis.
Starting from scale-dependent low-dimensional embeddings, we define cell-centered local neighborhoods at multiple resolutions.
For each local neighborhood, we construct a data-driven cellular sheaf that encodes local relationships among cells.
We then compute persistent sheaf Laplacians over sampled filtration intervals and extract spectral statistics that summarize the evolution of local relational structure across scales.
These spectral descriptors are aggregated into a unified feature vector for each cell and can be directly used in downstream learning tasks without additional model training.
We evaluate HSSE on twelve benchmark single-cell RNA-seq datasets covering diverse biological systems and data scales.
Under a consistent classification protocol, HSSE achieves competitive or improved performance compared with existing multiscale and classical embedding-based methods across multiple evaluation metrics.
The results demonstrate that sheaf spectral representations provide a robust and interpretable approach for single-cell RNA-seq data representation learning.
\end{abstract}


\section{Introduction}

Single-cell RNA-seq   technologies have significantly advanced the study of cellular heterogeneity by enabling measurement of gene expression at the resolution of individual cells~\cite{macosko2015highly}. 
These technologies produce increasingly large and complex datasets that reveal diverse biological states across tissues, developmental stages, and disease conditions~\cite{luecken2019current,cang2020inferring}. 
The challenge in single-cell RNA-seq data analysis now lies in developing computational and statistical methods to reliably analyze high-dimensional and noisy measurements~\cite{stegle2015computational, luecken2019current}.
The performance of downstream, such as classification, depends on how well the learned representations preserve the intrinsic geometry of the data~\cite{abdelaal2019comparison}.  

Single-cell RNA-seq data exhibit hierarchical structures across multiple scales~\cite{moon2019visualizing}.
Such structure is often characterized by an underlying manifold or intrinsic geometry~\cite{coifman2005geometric}.
Standard dimensionality reduction methods, including principal component analysis (PCA)~\cite{bro2014principal}, 
t-distributed stochastic neighbor embedding (t-SNE)~\cite{maaten2008visualizing}, non-negative matrix factorization (NMF)~\cite{lee1999learning}, and the Uniform Manifold Approximation and Projection (UMAP) algorithm~\cite{becht2019dimensionality},   map cells into low-dimensional Euclidean spaces.
Nonlinear embedding methods such as UMAP and tSNE typically rely on a single neighborhood resolution, which may limit their ability to capture hierarchical biological structure~\cite{becht2019dimensionality,maaten2008visualizing}.


Topological methods such as computational homology \cite{kaczynski2004computational} and persistent homology\cite{Edelsbrunner2010, zomorodian2004computing},  have provided new insights into single-cell datasets~\cite{beamer2025phd,huynh2024topological}.
These approaches characterize how connectivity and topological features evolve across filtrations, capturing multi-scale structural information that is difficult to obtain using conventional statistical or geometric approaches~\cite{Ghrist2008}. 
However, persistent homology primarily captures topological invariants and does not encode all geometric information in the data, which motivates the development of persistent topological Laplacians that integrate both topological and geometric features~\cite{ wei2025persistent}.

Persistent Laplacian-based spectral frameworks have been introduced to integrate topological and geometric information in a unified manner~\cite{wang2020persistent,wang2023persistent}.
The harmonic spectra recover topological invariants consistent with persistent homology, while the non-harmonic spectra capture additional geometric variations across filtrations~\cite{liu2024algebraic, memoli2022persistent,wei2025persistent,jones2025persistent}.
This new approach has been widely applied in computational biology and molecular science\cite{liu2021persistent,wee2025review,su2025topological}
Persistent Laplacians have been applied to single-cell RAN-seq data \cite{cottrell2024k} and sngle-cell spatial transcriptomics analysis \cite{cottrell2025multiscale}. Hodge  Laplacian and decomposition have been applied to single-cell velocity analysis \cite{su2024hodge}.  
However, many topological and spectral methods are constructed in a global manner, which may limit their ability to capture localized and hierarchical structures, particularly in large-scale datasets. 



Sheaf theory provides a principled mathematical framework for encoding local information in structured domains such as graphs and cell complexes~\cite{Edelsbrunner2010}. 
In particular, sheaf Laplacians relate local vector spaces through restriction maps that enforce compatibility among adjacent nodes, enabling the modeling of localized interactions in structured data~\cite{hansen2019toward}.
Persistent sheaf Laplacians were proposed to incorporate multiscale analysis into this local method \cite{wei2025persistentsheaf}. 
This local multiscale spectral theory has demonstrated its effectiveness  in learning from structured data~\cite{hayes2025persistent,hayes2026predicting}.
A software package has been developed to implement persistent topological Laplacians, including persistent sheaf Laplacians, for large scale computations \cite{jones2025petls}.
Persistent sheaf Laplacians are particularly relevant for single-cell RNA-seq data, where cellular heterogeneity often manifests across multiple scales and local neighborhoods \cite{cottrell2026computational}. However, existing approaches typically rely on fixed graph structures or additional supervision, and do not explicitly account for multi-scale geometric variation in the data. These observations suggest the importance of developing frameworks that integrate local modeling, multi-scale structure, and computational scalability for single-cell RNA-seq data analysis.

Motivated by these challenges, we develop a hierarchical spectral framework that integrates multi-level geometric information with sheaf-theoretic operators for single-cell RNA-seq data analysis.
The framework consists of three main components. First, each cell is associated with a collection of local neighborhoods at multiple resolutions.
Second, in each local neighborhood, we compute spectral features derived from sheaf Laplacians to capture additional structural variation.
Third, these features are aggregated across different resolutions and neighborhoods to form a unified representation for downstream tasks.

In this study, we construct a hierarchical representation through a multi-stage process.
We begin with a family of scale-dependent low-dimensional embeddings generated under varying neighborhood configurations, capturing structural information from fine-grained local patterns to broader biological organization.
For each embedding, we construct multiple $k$-nearest neighbor graphs to model local neighborhood structures at different scales.
On each resulting local patch, we sample a sequence of filtration values and compute persistent sheaf Laplacians to encode scale-dependent relational information.
From these operators, we extract eigenvalue-based spectral statistics that characterize the evolution of local structures across filtrations.
Finally, we aggregate spectral features across all representation scales, neighborhood sizes, and filtration values to obtain a unified feature vector for each cell, which can be used for classification and related downstream tasks.
We evaluate the proposed framework on twelve benchmark single-cell RNA-seq datasets spanning mouse, human, immune, neuronal, and pancreatic systems.
Across these datasets, the method achieves competitive classification performance and outperforms baseline manifold learning and topology-based methods in many cases.
In addition, the approach shows stable performance across different datasets and parameter settings. These results suggest that incorporating hierarchical and local spectral information provides an effective approach for constructing feature representations in complex single-cell RNA-seq datasets.


The remainder of this paper is organized as follows.
Section~2 reviews the mathematical background on cellular sheaves,
persistent sheaf Laplacians, and spectral operators.
Section~3 presents the hierarchical sheaf spectral framework in detail.
Section~4 reports experimental results in twelve single-cell RNA-seq datasets and includes an ablation analysis.
Section~5 discusses the effects of multi-scale design choices,
the contribution of topological information, and the robustness of the proposed framework.
Section~6 concludes the paper.

\section{Mathematical Foundations}

This section establishes the mathematical framework and notation adopted in this work.
We outline the key components needed in the subsequent development,
including cellular sheaves defined on simplicial complexes,
their associated cochain constructions and sheaf Laplacians,
as well as filtration procedures and the resulting spectral operators.
Together, these elements form the basis of the persistent sheaf Laplacian framework
and support the subspace representations used for comparison in later sections.

\begin{table}[h]
  \centering
  \begin{tabular}{c|l}
    \toprule
    \textbf{Symbol} & \textbf{Description} \\
    \midrule
    $\mathbb{R}$ & Real number field \\
    $\mathbb{R}^n$ & $n$ dimensional real vector space \\
    $\mathbb{R}^{n \times n}$ & Set of all real $n \times n$ matrices \\

    $X$ & Single cell data matrix with rows representing cells \\
    $N$ & Number of cells in the dataset \\
    $d$ & Ambient feature dimension of the data \\

    $\mathcal{S}$ & The set of scale parameters
    \\
    $\mathcal{E}_s$ & Low dimensional embedding of the dataset at scale $s$ \\

    $\mathcal{K}$ & The
    set of neighborhood sizes
    \\
    $\mathcal{N}_k(i)$ & Set of $k$ nearest neighbors of cell $i$ under embedding $\mathcal{E}_s$ \\

    $\mathcal{P}_{i}^{s,k}$ & Local neighborhood patch centered at cell $i$ constructed at scale $s$ with $k$ neighbors \\

    $\mathcal{F}$ & Filtration parameter used in persistent sheaf Laplacian construction \\
    $\{\mathcal{F}_\ell\}_{\ell=1}^L$ & Sampled filtration values \\

    $\mathcal{S}_{i}^{s,k}$ & Cellular sheaf defined on the local patch 
    \\
    $\mathcal{L}^{q}_{\mathcal{S}_{i}^{s,k}}(\mathcal{F}_\ell)$ & Persistent sheaf Laplacian of degree $q$ at filtration $\mathcal{F}_\ell$ \\

    $\lambda_j(\mathcal{F}_\ell)$ & $j$th eigenvalue of the sheaf Laplacian at filtration $\mathcal{F}_\ell$ \\
    $\boldsymbol{\lambda}(\mathcal{F}_\ell)$ & Eigenvalue spectrum at filtration $\mathcal{F}_\ell$ \\

    $\phi(\cdot)$ & Spectral feature extraction function \\
    $\mathbf{z}_i$ & Aggregated spectral feature vector of cell $i$ \\

    $y_i$ & Ground truth label of cell $i$ \\
    $\hat{y}_i$ & Predicted label of cell $i$ \\
    \bottomrule
  \end{tabular}
  \caption{Notations used in the paper.}
  \label{tab:notation}
\end{table}

\subsection{Simplices and Simplicial Complexes}

Simplices and simplicial complexes serve as a combinatorial language
for encoding higher-order relationships in a finite set of points.
They play a central role in topological data analysis
by providing a discrete structure that captures geometric and relational information
\cite{Edelsbrunner2010}.

Let $u_0,u_1,\dots,u_k \in \mathbb{R}^d$.
A point $x$ is said to be an \emph{affine combination} of these points
if it can be written as
\[
x=\sum_{i=0}^k \lambda_i u_i,
\quad \text{with} \quad \sum_{i=0}^k \lambda_i = 1.
\]
The points $u_0,\dots,u_k$ are called \emph{affinely independent}
if such a representation is unique, that is,
whenever
\[
\sum_{i=0}^k \lambda_i u_i = \sum_{i=0}^k \mu_i u_i,
\]
it follows that $\lambda_i=\mu_i$ for all $i$.
This condition is equivalent to requiring that
$\{u_i - u_0\}_{i=1}^k$ are linearly independent.

Given affinely independent points $u_0,\dots,u_k$,
their convex hull
\[
\sigma = \mathrm{conv}\{u_0,u_1,\dots,u_k\}
\]
defines a \emph{$k$-simplex}.
Any simplex generated by a non-empty subset of these vertices
is referred to as a \emph{face} of $\sigma$.

Based on these constructions, we define a simplicial complex.

\begin{definition}[Simplicial Complex]
A \emph{simplicial complex} $K$ is a finite collection of simplices satisfying:
\begin{enumerate}
    \item if $\sigma \in K$, then every face $\tau \subseteq \sigma$ also belongs to $K$;
    \item for any $\sigma, \tau \in K$, the intersection $\sigma \cap \tau$
    is either empty or a face shared by both.
\end{enumerate}
We denote by $K^{(k)}$ the set of all $k$-simplices in $K$,
and define the dimension of $K$ as
\[
\dim K = \max_{\sigma \in K} \dim \sigma.
\]
\end{definition}

We use the notation $\sigma \le \tau$ to indicate that $\sigma$ is a face of $\tau$.
Simplicial complexes can be viewed as a particular instance
of regular cell complexes \cite{lundell2012topology,curry2014sheaves,hansen2019toward}.
 
\subsection{Cellular Sheaves on Simplicial Complexes}

In this work, we focus on cellular sheaves defined in simplicial complexes.
This choice is motivated by two considerations.
First, simplicial complexes form a subclass of regular cell complexes
and therefore admit the standard theory of cellular sheaves.
Second, simplicial complexes arise naturally in data-driven settings,
such as those constructed from pairwise distances or neighborhood graphs,
making them particularly suitable for computational and applied purposes.
Our presentation follows the simplicial formulation adopted in
\cite{wei2025persistent}, which enables the construction of persistent
sheaf Laplacians in a multi-scale setting.

 \begin{definition}[Cellular Sheaf]
Let $X$ be a simplicial complex with face relation denoted by $\sigma \le \tau$.
A \emph{cellular sheaf} $\mathcal{F}$ on $X$ consists of:
\begin{enumerate}
    \item for each simplex $\sigma \in X$, a finite-dimensional vector space
    $\mathcal{F}(\sigma)$;
    \item for each pair $\sigma \le \tau$, a linear map
    \[
    \mathcal{F}_{\sigma\le\tau} : \mathcal{F}(\sigma) \to \mathcal{F}(\tau),
    \]
\end{enumerate}
such that for any $\rho \le \sigma \le \tau$,
\[
\mathcal{F}_{\rho\le\tau}
=
\mathcal{F}_{\sigma\le\tau} \circ \mathcal{F}_{\rho\le\sigma},
\qquad
\mathcal{F}_{\sigma\le\sigma} = \mathrm{id}.
\]
The space $\mathcal{F}(\sigma)$ is called the stalk over $\sigma$,
and $\mathcal{F}_{\sigma\le\tau}$ is called a restriction map.
\end{definition}


\subsection{Sheaf Cochains and Coboundary Operators}

Let $\mathcal{F}$ be a cellular sheaf defined on a simplicial complex $X$.
From $\mathcal{F}$, one can construct a cochain sequence in which the linear
structure is induced by the restriction maps of the sheaf.
This construction will be used later to define sheaf Laplacians.

For each integer $k \ge 0$, the space of $k$-cochains is given by
\[
C^k(X;\mathcal{F})
=
\bigoplus_{\sigma \in X^{(k)}} \mathcal{F}(\sigma),
\]
where $X^{(k)}$ denotes the set of $k$-simplices of $X$.
Thus, a $k$-cochain assigns to every $k$-simplex $\sigma$
a vector in the corresponding stalk $\mathcal{F}(\sigma)$.

The coboundary operator
\[
\delta_k : C^k(X;\mathcal{F}) \to C^{k+1}(X;\mathcal{F})
\]
is defined using the restriction maps associated with the sheaf.
More precisely, for each relation $\sigma \le \tau$ with
$\sigma \in X^{(k)}$ and $\tau \in X^{(k+1)}$,
the corresponding component of $\delta_k$ is given by
\[
\mathcal{F}_{\sigma \le \tau} : \mathcal{F}(\sigma) \to \mathcal{F}(\tau),
\]
together with a sign determined by an orientation of simplices.
Combining these maps over all such pairs $(\sigma,\tau)$
defines a linear operator $\delta_k$ between the cochain spaces.

These operators satisfy the relation
\[
\delta_{k+1} \circ \delta_k = 0,
\]
and therefore form a sequence
\[
0 \longrightarrow C^0(X;\mathcal{F})
\stackrel{\delta_0}{\longrightarrow} C^1(X;\mathcal{F})
\stackrel{\delta_1}{\longrightarrow} C^2(X;\mathcal{F})
\longrightarrow \cdots,
\]
which is referred to as the sheaf cochain complex.

 \subsection{Persistent Sheaf Laplacians}

To capture multi-scale structure, cellular sheaves and their associated
operators are considered over a filtration of simplicial complexes.
This leads to the notion of a persistent sheaf Laplacian, which extends
the classical persistent Laplacian by incorporating sheaf-valued data
and restriction maps \cite{wei2025persistent}.

Let $\{X_t\}_{t \in \mathcal{T}}$ be a filtration of simplicial complexes,
that is, a nested family
\[
X_{t_1} \subseteq X_{t_2} \subseteq \cdots \subseteq X_{t_m},
\]
indexed by a finite set of scale parameters $\mathcal{T}$.
Let $\mathcal{F}$ be a cellular sheaf defined on each $X_t$,
with compatible stalks and restriction maps across the filtration.

For each scale $t$ and each degree $k \ge 0$, the sheaf cochain complex
$(C^\bullet(X_t;\mathcal{F}), \delta_\bullet(t))$
induces a family of scale-dependent operators.
The \emph{$k$-th persistent sheaf Laplacian} at scale $t$ is defined as
\[
L_k(t)
=
\delta_k(t)^{\top}\delta_k(t)
+
\delta_{k-1}(t)\delta_{k-1}(t)^{\top},
\]
where $\delta_k(t)$ denotes the coboundary operator on $X_t$,
and the adjoint is taken with respect to the standard inner product
on the cochain spaces.
By convention, the second term is omitted when $k=0$.

For each fixed $t$, the operator $L_k(t)$ is symmetric and positive semidefinite,
and hence admits a real, nonnegative spectrum.
As the scale parameter varies, the eigenvalues of $L_k(t)$ evolve with the
filtration, yielding a family of spectral quantities that encode
the interaction between local sheaf data and the underlying multi-scale
combinatorial structure.

To aggregate spectral information over a range of filtration values,
one considers persistent operators defined on filtration intervals.
Following \cite{wei2025persistent}, let $a,b \in \mathcal{T}$ with $a \le b$.
The \emph{$k$-th $(a,b)$-persistent sheaf Laplacian} is defined as an operator
\[
L^{a,b}_k : C^k(X_a;\mathcal{F}) \longrightarrow C^k(X_a;\mathcal{F}),
\]
given by
\[
L^{a,b}_k
=
\delta^{a,b}_k \, (\delta^{a,b}_k)^{\top}
+
\delta^{a}_{k-1} \, (\delta^{a}_{k-1})^{\top}.
\]
Here, $\delta^{a}_{k}$ denotes the coboundary operator on $X_a$, and
$\delta^{a,b}_k$ is the persistent coboundary operator induced by the
inclusion $X_a \hookrightarrow X_b$.
The adjoint is taken with respect to the standard inner product on the
cochain spaces.

The operator $L^{a,b}_k$ is symmetric and positive semidefinite, and its
spectrum captures sheaf-consistent features that persist across the
filtration interval $[a,b]$.
Varying the interval endpoints $(a,b)$ yields a family of persistent
spectral descriptors associated with the underlying filtration.

In this work, we use the eigenvalues of persistent sheaf Laplacians
as scale-dependent features.
This spectral representation provides a principled way to summarize
multi-scale geometric and relational information, and forms the basis
of the hierarchical embedding framework introduced in the next section.




\section{Hierarchical Sheaf Spectral Embedding Framework}

In this section, we present a hierarchical spectral framework based on
persistent sheaf Laplacians for single-cell data analysis.
The framework is organized into a three-level pipeline:
multi-scale representations are constructed first,
local neighborhood complexes are then built within each representation,
and persistent sheaf Laplacian spectra are finally computed over sampled
filtration intervals.
Eigenvalues obtained under different combinations of representation scale,
neighborhood size, and filtration parameter are aggregated into feature
vectors, which are subsequently used for downstream tasks such as
classification.
This section formalizes each component of the framework and clarifies
their roles within the overall pipeline.

\subsection{Overall Framework}

Let $X \in \mathbb{R}^{m \times n}$ denote a single-cell gene expression matrix,
where $m$ is the number of cells and $n$ is the number of measured genes.
The $i$-th row $x_i \in \mathbb{R}^n$ represents the expression profile of
the $i$-th cell.
The objective of the proposed framework is to construct, for each cell,
a feature representation that captures multi-scale relational structure
through persistent sheaf Laplacian spectra.

\begin{figure}[t]
    \centering
    \includegraphics[width=\textwidth]{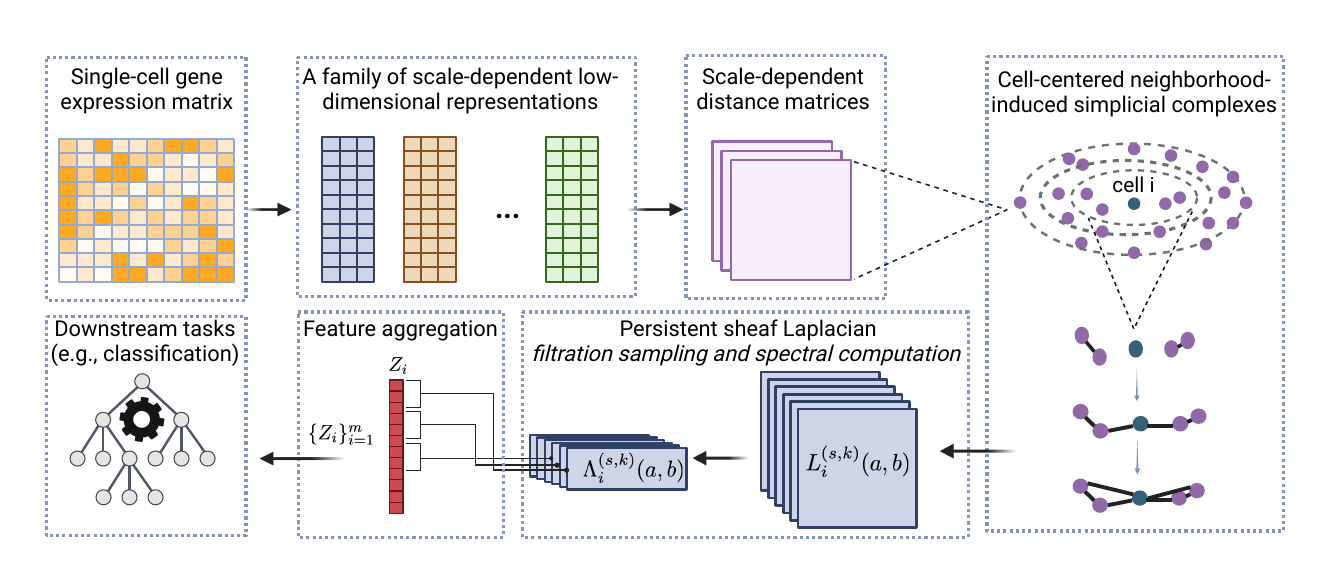}
    \caption{Overview of the proposed cell-centered persistent sheaf Laplacian framework for single-cell data analysis.
Starting from the single-cell gene expression matrix, a family of scale-dependent low-dimensional representations is constructed to capture cellular structures at different resolutions.
Based on these representations, scale-dependent distance matrices are computed to quantify pairwise relationships between cells.
For each cell $i$, a cell-centered neighborhood is identified from the distance matrices, and neighborhood-induced simplicial complexes are constructed to encode local topological structures around the cell.
On these cell-centered simplicial complexes, persistent sheaf Laplacians are computed via filtration sampling and spectral computation, yielding spectral information that characterizes the underlying sheaf-enhanced topological structures.
These spectral features are aggregated to form a feature vector $Z_i$ associated with cell $i$.
Collecting the feature vectors for all cells gives $\{Z_i\}_{i=1}^m$, where $m$ denotes the total number of cells, which serves as input for downstream learning tasks, e.g., classification.}
    \label{fig:framework}
\end{figure}

The framework follows a hierarchical construction composed of several
nested stages.
An overview is provided in Fig.~\ref{fig:framework}.
We summarize the main steps below.

\textbf{Step 1 (A family of scale-dependent low-dimensional representations).}
Starting from the gene expression matrix
\[
X \in \mathbb{R}^{m \times n},
\]
where $m$ denotes the number of cells and $n$ the number of genes,
we construct a family of scale-dependent low-dimensional representations
via dimension reduction.
Let
\[
\mathcal{S} = \{s_1, s_2, \dots, s_p\}
\]
denote a finite set of scale parameters.
We consider a family of mappings
\[
\{\Phi^{(s)} : \mathbb{R}^{m \times n} \rightarrow \mathbb{R}^{m \times d_s}
\mid s \in \mathcal{S}\},
\]
where where each mapping $\Phi^{(s)}$ produces a reduced representation
associated with scale parameter $s$, with $d_s \ll n$.
These mappings are applied to the entire dataset and are generated
independently across scales, without aggregation at this stage.
Applying $\Phi^{(s)}$ yields a low-dimensional representation
\[
Y^{(s)} = \Phi^{(s)}(X) \in \mathbb{R}^{m \times d_s},
\]
where each row of $Y^{(s)}$ corresponds to the reduced feature vector
of a cell under scale $s$.

\textbf{Step 2 (Scale-dependent distance matrices).}
For each scale-dependent representation
\[
Y^{(s)} = \Phi^{(s)}(X) \in \mathbb{R}^{m \times d_s}, \qquad s \in \mathcal{S},
\]
we equip the embedded space with a distance function
$d^{(s)}(\cdot,\cdot)$ and compute the corresponding pairwise distance matrix
\[
D^{(s)} \in \mathbb{R}^{m \times m}, \qquad
D^{(s)}_{ij} = d^{(s)}\!\left(Y^{(s)}_i,\,Y^{(s)}_j\right),
\]
where $Y^{(s)}_i \in \mathbb{R}^{d_s}$ denotes the $i$-th row of $Y^{(s)}$,
corresponding to the reduced feature vector of the $i$-th cell under scale $s$.
The family of distance matrices
\[
\{D^{(s)} \mid s \in \mathcal{S}\}
\]
provides the geometric input for neighborhood construction and filtration
in subsequent steps.

\textbf{Step 3 (Cell-centered neighborhood-induced simplicial complexes).}
Within each scale-dependent representation, local neighborhood structures
are constructed separately for each cell with respect to the corresponding
distance matrix $D^{(s)}$.
For a prescribed set of neighborhood sizes $\mathcal{K}$, we define,
for each cell $i$ and each $k \in \mathcal{K}$, a $k$-nearest-neighbor
neighborhood
\[
\mathcal{N}^{(s,k)}_i \subseteq \{1,\dots,m\},
\]
selected according to the $k$ smallest values of $D^{(s)}_{ij}$ over
$j \neq i$, with the center cell always included, i.e.,
$i \in \mathcal{N}^{(s,k)}_i$.
Each neighborhood $\mathcal{N}^{(s,k)}_i$ induces a local simplicial complex
$K^{(s,k)}_i$ whose vertex set is given by $\mathcal{N}^{(s,k)}_i$.
Simplices in $K^{(s,k)}_i$ are formed according to the simplicial constructions
reviewed in Section~2, ensuring consistency with the face relations of a
simplicial complex.

\textbf{Step 4 (Cell-centered persistent sheaf Laplacian spectra).}
For each cell-centered simplicial complex $K^{(s,k)}_i$, persistent sheaf
Laplacians are computed to extract spectral information.
Following the definitions in Section~2, a filtration
\[
\{K^{(s,k)}_{i,t}\}_{t \in [a,b]}
\]
is generated by sampling filtration parameters over an interval $[a,b]$,
where different choices of the interval endpoints $(a,b)$ correspond to
distinct filtration sampling schemes.
On each filtered complex, the corresponding $(a,b)$-persistent sheaf
Laplacian
\[
L^{(s,k)}_i(a,b)
\]
is constructed, and its eigenvalues are computed.

\textbf{Step 5 (Feature aggregation and downstream tasks).}
For each cell $i$, and for each combination of representation scale $s$,
neighborhood size $k$, and filtration interval $(a,b)$, summary statistics
of the eigenvalues, including the sum, mean, maximum, minimum, and standard
deviation, are extracted to form feature vectors.
All feature vectors associated with a given cell are concatenated to produce
a cell-level feature representation.
These representations are subsequently used for downstream learning tasks,
such as classification, at the dataset level.

\subsection{Distance-Induced Filtrations and Interval Sampling}

The persistent sheaf Laplacian constructions considered in this work
are based on filtrations derived directly from distance information.
In this subsection, we describe how distance matrices induce local
simplicial complexes via Rips constructions, and how filtration
intervals are selected for $(a,b)$-persistent sheaf Laplacian analysis.

\subsubsection{Distance-Induced Rips Complexes}

Let $D \in \mathbb{R}^{m \times m}$ denote a symmetric distance matrix
defined on a local cell-centered neighborhood.
To construct a simplicial complex from $D$, we employ a Vietoris--Rips
construction.
Specifically, given a distance threshold $r \ge 0$, the Rips complex
$R(D,r)$ is defined by including a simplex $\sigma = \{i_0,\dots,i_k\}$
if and only if
\[
D_{i_p i_q} \le r
\qquad
\text{for all } 0 \le p < q \le k.
\]

In our setting, the maximum filtration scale is determined directly
from the observed distances.
Let
\[
r_{\max} = \max \{ D_{ij} \mid D_{ij} > 0 \}.
\]
When no explicit upper bound is specified, $r_{\max}$ is used as the
maximum edge length for the Rips construction.
This choice ensures that the resulting simplicial complex captures
the full range of nontrivial pairwise relationships present in the
local distance matrix, without introducing extraneous scale parameters.

To control computational complexity, the dimension of the Rips complex
is truncated at a prescribed maximum dimension.
Throughout this work, we restrict attention to low-dimensional
simplicial complexes, which is sufficient for the persistent sheaf
Laplacian analysis considered here.

\subsubsection{$(a,b)$-Filtration Intervals}

The construction of persistent sheaf Laplacians in this work is based on
interval sampling over distance-induced filtrations.
Given a local distance matrix $D$ and its maximum nonzero entry
\[
b_{\max} = \max \{ D_{ij} \mid D_{ij} > 0 \},
\]
we generate a collection of filtration intervals
\[
\{[a_\ell, b_\ell]\}_{\ell=1}^L \subseteq [0, b_{\max}],
\]
which partition the filtration range into contiguous subintervals.

Specifically, the interval $[0,b_{\max}]$ is divided into $L$ segments
with uniformly spaced breakpoints,
\[
0 = a_1 < b_1 = a_2 < b_2 = \cdots < a_L < b_L = b_{\max}.
\]
Each interval $[a_\ell, b_\ell]$ isolates geometric information associated
with a specific range of distance scales.
For each interval, an $(a,b)$-persistent sheaf Laplacian is constructed
on the corresponding filtered simplicial complex.

This segment-based sampling strategy provides localized spectral information
across distance scales and is used throughout all experiments in this work.
\subsection{Cell-Centered Cellular Sheaves with Interpretable Restriction Maps}

In this subsection, we introduce cell-centered cellular sheaves equipped
with interpretable restriction maps tailored to single-cell data.
Unlike abstract sheaf constructions, the restriction maps used here are
designed to reflect local geometric and relational structure within each
cell-centered neighborhood.
This design enables the persistent sheaf Laplacian to encode meaningful
local variation while remaining compatible with the general sheaf-theoretic
framework reviewed in Section~2.

\subsubsection{Cell-Centered Cellular Sheaves}

We begin by defining the cellular sheaf structures associated with each
cell-centered simplicial complex constructed in Section~3.1.
The key feature of our framework is that sheaves are not defined globally
on a single complex, but locally and independently for each center cell,
scale parameter, and neighborhood size.

Recall that, for each cell $i$, scale parameter $s \in \mathcal{S}$,
and neighborhood size $k \in \mathcal{K}$, we construct a local simplicial
complex
$
K_i^{(s,k)},
$
whose vertex set consists of the indices of cells in the
$k$-nearest-neighbor neighborhood of cell $i$ under scale $s$.
These local complexes serve as the base spaces on which cellular sheaves
are defined.

\begin{definition}[Cell-Centered Cellular Sheaf]
For each triple $(i,s,k)$, let $K_i^{(s,k)}$ be the associated
cell-centered simplicial complex.
A \emph{cell-centered cellular sheaf}
$
\mathcal{F}_i^{(s,k)}$
on $K_i^{(s,k)}$ consists of the following assignments:
\begin{enumerate}
    \item to each simplex $\sigma \in K_i^{(s,k)}$, a finite-dimensional
    real vector space $\mathcal{F}_i^{(s,k)}(\sigma)$, called the
    \emph{stalk} over $\sigma$;
    \item to each face relation $\sigma \le \tau$ in $K_i^{(s,k)}$, a linear
    map
    \[
    \mathcal{F}_{i,\sigma \le \tau}^{(s,k)} :
    \mathcal{F}_i^{(s,k)}(\sigma) \longrightarrow
    \mathcal{F}_i^{(s,k)}(\tau),
    \]
    called the \emph{restriction map}.
\end{enumerate}
These assignments satisfy the standard sheaf consistency conditions:
\[
\mathcal{F}_{i,\sigma \le \sigma}^{(s,k)} = \mathrm{id},
\qquad
\mathcal{F}_{i,\rho \le \tau}^{(s,k)}
=
\mathcal{F}_{i,\sigma \le \tau}^{(s,k)}
\circ
\mathcal{F}_{i,\rho \le \sigma}^{(s,k)}
\quad
\text{for all } \rho \le \sigma \le \tau.
\]
\end{definition}

Throughout this work, we focus on cellular sheaves with
one-dimensional stalks,
\[
\mathcal{F}_i^{(s,k)}(\sigma) \cong \mathbb{R}
\quad \text{for all } \sigma \in K_i^{(s,k)}.
\]
Under this assumption, each restriction map reduces to multiplication by
a scalar.
The specific choice of these scalars, which encodes local geometric and
relational information, is discussed in the next subsection.

\subsubsection{Interpretable Restriction Maps}

We now define the restriction maps used in the cell-centered cellular
sheaves introduced in the previous subsection.
Throughout this subsection, we consider cellular sheaves with
one-dimensional stalks,
\[
\mathcal{F}_i^{(s,k)}(\sigma) \cong \mathbb{R},
\]
so that each restriction map is necessarily a linear map represented by
multiplication by a scalar.

Let $D^{(s)}$ denote the distance matrix associated with scale $s$ on the
vertex set of the local simplicial complex $K_i^{(s,k)}$.
Let $\eta > 0$ be a scale parameter determined from the local distance
distribution.
Specifically, we set
\[
\eta
=
\operatorname{median}
\left\{
D^{(s)}_{uv} \;\middle|\;
D^{(s)}_{uv} > 0
\right\},
\]
that is, the median of the nonzero pairwise distances on the vertex set of
$K_i^{(s,k)}$.
This choice provides a data-adaptive normalization of distances and
avoids the introduction of additional tuning parameters.

For any pair of vertices $u,v \in K_i^{(s,k)}$, we write
\[
d = D^{(s)}_{uv}
\]
for their pairwise distance under scale $s$.
We define a geometry-based kernel
\[
k(d) = \exp\!\left(-\frac{d^2}{\eta^2}\right),
\]
which assigns larger weights to pairs of vertices that are closer in the
embedded space.

To encode asymmetric roles within each local neighborhood, we associate
to each vertex $u$ a label value
\[
\ell(u) \in \mathcal{L},
\]
where $\mathcal{L}$ denotes a prescribed label space.
The discrepancy between two vertices $u$ and $v$ is measured by a
nonnegative function
\[
m(u,v) \ge 0.
\]
This formulation allows label information to be incorporated in a flexible
and extensible manner.

In the present work, we adopt a simple binary labeling scheme
$\mathcal{L} = \{0,1\}$, where $\ell(u)=0$ denotes the center cell and
$\ell(u)=1$ denotes neighboring cells, and define
\[
m(u,v) = |\ell(u)-\ell(v)|.
\]
This choice yields an interpretable and minimal distinction between the
center cell and its neighbors while remaining fully compatible with the
general framework.
More general label spaces and discrepancy functions can be incorporated
without altering the sheaf-theoretic construction.

Given a nonnegative parameter $\alpha$, we introduce a label modulation
factor
\[
\exp(-\alpha\, m(u,v)),
\]
which downweights restriction maps between vertices with larger label
discrepancy.
When $\alpha = 0$, the restriction maps depend purely on geometric
information.

\begin{definition}[Restriction Maps]
Let $\mathcal{F}_i^{(s,k)}$ be a cell-centered cellular sheaf on
$K_i^{(s,k)}$.
The restriction maps are defined as follows.

\begin{enumerate}
\item \textbf{Identity.}
For every simplex $\sigma \in K_i^{(s,k)}$, the restriction map
\[
\rho_{\sigma \le \sigma} :
\mathcal{F}_i^{(s,k)}(\sigma)
\longrightarrow
\mathcal{F}_i^{(s,k)}(\sigma)
\]
is defined as the identity map,
\[
\rho_{\sigma \le \sigma}(x) = x.
\]

\item \textbf{Codimension--one relations.}
For codimension--one face relations, the restriction maps are given by:
\begin{itemize}
\item \emph{Vertex to edge.}
For a vertex $u$ and an incident edge $(u,v)$,
\[
\rho_{u \le (u,v)}(x)
=
k\!\left(D^{(s)}_{uv}\right)
\exp\!\bigl(-\alpha\, m(u,v)\bigr)\, x.
\]

\item \emph{Edge to triangle.}
For an edge $(u,v)$ and an incident triangle $(u,v,w)$,
\[
\rho_{(u,v) \le (u,v,w)}(x)
=
\frac{1}{2}
\left[
k\!\left(D^{(s)}_{uw}\right)\exp\!\bigl(-\alpha\, m(u,w)\bigr)
+
k\!\left(D^{(s)}_{vw}\right)\exp\!\bigl(-\alpha\, m(v,w)\bigr)
\right] x.
\]
\end{itemize}

\item \textbf{Higher-codimension relations.}
For any face relation $\rho \le \tau$ with
$\dim(\tau)-\dim(\rho) \ge 2$, the restriction map is defined by
composition.
That is, for any chain of simplices
\[
\rho = \sigma_0 \le \sigma_1 \le \cdots \le \sigma_r = \tau,
\qquad
\dim(\sigma_{j+1}) = \dim(\sigma_j)+1,
\]
we set
\[
\rho_{\rho \le \tau}
=
\rho_{\sigma_{r-1} \le \sigma_r}
\circ \cdots \circ
\rho_{\sigma_0 \le \sigma_1}.
\]
\end{enumerate}
\end{definition}

\begin{prop}[Sheaf consistency]
The restriction maps defined above satisfy the axioms of a cellular
sheaf.
In particular, for all simplices $\rho \le \sigma \le \tau$ in
$K_i^{(s,k)}$,
\[
\rho_{\sigma \le \sigma} = \mathrm{id},
\qquad
\rho_{\rho \le \tau}
=
\rho_{\sigma \le \tau} \circ \rho_{\rho \le \sigma}.
\]
\end{prop}

\begin{proof}
By construction, each restriction map is a scalar multiple of the identity
on $\mathbb{R}$ and hence linear.
The identity condition holds by definition.

For the composition property, restriction maps on higher-codimension face
relations are defined as compositions of codimension--one restriction
maps.
Thus, for $\rho \le \sigma \le \tau$ and any
$x \in \mathcal{F}_i^{(s,k)}(\rho) \cong \mathbb{R}$,
\[
\rho_{\sigma \le \tau}\!\left(\rho_{\rho \le \sigma}(x)\right)
=
\rho_{\rho \le \tau}(x),
\]
which establishes the required consistency.
\end{proof}

\begin{remark}[Interpretability]
The restriction maps combine geometric proximity and vertex label
information in a transparent manner.
The kernel term emphasizes locally coherent neighborhoods, while the
label modulation controls the relative influence of the center cell and
its neighbors.
The parameter $\alpha$ provides a continuous interpolation between a
purely geometric sheaf ($\alpha=0$) and a center-aware sheaf structure
($\alpha>0$), allowing the effect of label information to be tuned
explicitly.
\end{remark}


\subsection{Scaling Invariance of the Persistent Sheaf Laplacian}

We record a simple but important scaling invariance of the proposed
construction.
This property clarifies that the spectral information extracted by
persistent sheaf Laplacians is governed primarily by the \emph{relative
pattern} of restriction maps, rather than their absolute magnitudes.
Such an invariance is desirable in the present framework, where
restriction weights are constructed from data-adaptive kernels and may
exhibit unavoidable global rescaling across different local neighborhoods
and representation scales.

\begin{prop}[Scaling invariance of PSL eigenvectors]
\label{prop:scaling-invariance}
Fix a cell-centered simplicial complex $K_i^{(s,k)}$ and a filtration
interval $(a,b)$.
Let $\mathcal{F}$ be a cellular sheaf on $K_i^{(s,k)}$ with one-dimensional
stalks, and let $\{L_r(a,b)\}_{r\ge 0}$ denote the corresponding
$(a,b)$-persistent sheaf Laplacian operators.
If all restriction maps are uniformly rescaled by a constant factor
$c>0$, namely
\[
\rho_{\sigma\le\tau} \;\mapsto\; c\,\rho_{\sigma\le\tau}
\qquad \text{for all face relations } \sigma\le\tau,
\]
then for each degree $r$ the Laplacian rescales as
\[
L_r(a,b) \;\mapsto\; c^2\,L_r(a,b).
\]
Consequently, the eigenspaces (and hence eigenvectors, up to basis choice
within eigenspaces) of $L_r(a,b)$ are unchanged, while eigenvalues are
scaled by the factor $c^2$.
\end{prop}

\begin{proof}
For each fixed filtration parameter $t$ in the interval $(a,b)$, the
$c$-rescaling multiplies every codimension--one restriction map by $c$.
Since the coboundary operator $\delta_r(t)$ is assembled linearly from
these codimension--one restriction maps (up to orientation signs), it
follows that
\[
\delta_r(t) \;\mapsto\; c\,\delta_r(t)
\qquad \text{for all } r.
\]
Therefore, the associated sheaf Laplacian at scale $t$ satisfies
\[
L_r(t)
=
\delta_r(t)^{\top}\delta_r(t)
+
\delta_{r-1}(t)\delta_{r-1}(t)^{\top}
\;\mapsto\;
c^2\,\delta_r(t)^{\top}\delta_r(t)
+
c^2\,\delta_{r-1}(t)\delta_{r-1}(t)^{\top}
=
c^2\,L_r(t).
\]
The same scaling relation holds for the $(a,b)$-persistent sheaf Laplacian
operator $L_r(a,b)$ obtained from the filtration over $t\in(a,b)$.
Finally, multiplying a symmetric matrix by a positive scalar rescales its
eigenvalues by that scalar and leaves its eigenspaces unchanged.
\end{proof}

\begin{remark}[Implications for the proposed framework]
Proposition~\ref{prop:scaling-invariance} has several practical
implications for the proposed hierarchical framework.
First, it shows that global rescaling of restriction weights—arising, for
example, from differences in local distance distributions or numerical
normalization—does not alter the spectral directions captured by the
persistent sheaf Laplacian.
Second, this invariance is consistent with the use of data-adaptive kernel
scales, such as the median-based choice of $\eta$, which may induce
different overall weight magnitudes across cell-centered neighborhoods.
Finally, the result explains why the framework is primarily sensitive to
relative geometric and relational structure, and helps interpret the
robustness of the extracted features to certain parameter variations
observed in the experiments.
\end{remark}

\subsection{Algorithm Summary}

We conclude this section with an algorithmic summary of the proposed
hierarchical sheaf spectral framework.
Given a single-cell gene expression matrix, the method constructs a family
of scale-dependent representations, builds cell-centered neighborhood
complexes, computes $(a,b)$-persistent sheaf Laplacian spectra on the
resulting filtrations, and aggregates spectral statistics into
cell-level feature vectors for downstream learning tasks.
Algorithm~\ref{alg:hsse} summarizes the proposed hierarchical sheaf
spectral embedding framework.
\begin{algorithm}[ht]
\caption{Hierarchical Sheaf Spectral Embedding (HSSE)}
\label{alg:hsse}
\hrule\vspace{0.3em}
\begin{algorithmic}[1]
    \State \textbf{Input:} gene expression matrix $X\in\mathbb{R}^{m\times n}$;
    scales $\mathcal{S}=\{s_1,\dots,s_p\}$;
    neighborhood sizes $\mathcal{K}=\{k_1,\dots,k_q\}$;
    filtration intervals $\mathcal{I}=\{(a_1,b_1),\dots,(a_r,b_r)\}$ (segments);
    label strength $\alpha\ge 0$;
    maximal simplex dimension $d_{\max}=2$.
    \State \textbf{Output:} feature matrix $Z\in\mathbb{R}^{D\times m}$ with column $\{Z_i\}_{i=1}^m$.

    \State Initialize $Z_i \gets [\,]$ for all $i=1,\dots,m$ \Comment{empty feature lists}
    \ForAll{$s\in\mathcal{S}$}
        \State Compute scale-dependent representation $Y^{(s)} \gets \Phi^{(s)}(X)\in\mathbb{R}^{m\times d_s}$
        \State Compute distance matrix $D^{(s)}\in\mathbb{R}^{m\times m}$ on rows of $Y^{(s)}$
        \For{$i=1$ to $m$}
            \ForAll{$k\in\mathcal{K}$}
                \State Construct $k$-NN index set $\mathcal{N}^{(s,k)}_i$ using $D^{(s)}$ (ensure $i\in\mathcal{N}^{(s,k)}_i$)
                \State Extract local submatrix $D^{(s,k)}_i \gets D^{(s)}[\mathcal{N}^{(s,k)}_i,\mathcal{N}^{(s,k)}_i]$
                \State Set $r_{\max}\gets \max\{D^{(s,k)}_i(u,v):D^{(s,k)}_i(u,v)>0\}$ \Comment{max nonzero distance}
                \State Build a Rips complex $K^{(s,k)}_i$ from $D^{(s,k)}_i$ with max edge length $r_{\max}$ and dimension $d_{\max}$
                \State Set $\eta \gets \operatorname{median}\{D^{(s,k)}_i(u,v):D^{(s,k)}_i(u,v)>0\}$ \Comment{kernel scale}
                \State Assign center--neighbor labels $\ell$ on vertices of $K^{(s,k)}_i$
                \State Define restriction maps $\rho$ using $k(d)=\exp(-d^2/\eta^2)$ and label modulation $\exp(-\alpha\,m(u,v))$
                \ForAll{$(a,b)\in\mathcal{I}$}
                    \State Compute PSL spectra $\Lambda^{(s,k)}_i(a,b)$ on $K^{(s,k)}_i$ over interval $(a,b)$
                    \State Compute summary statistics $\psi(\Lambda)\gets(\mathrm{sum},\mathrm{mean},\mathrm{max},\mathrm{min},\mathrm{std})$
                    \State Append $\psi(\Lambda^{(s,k)}_i(a,b))$ to $Z_i$
                \EndFor
            \EndFor
        \EndFor
    \EndFor
    \State \Return $Z$
\end{algorithmic}
\vspace{0.3em}\hrule
\end{algorithm}

\begin{remark} 
The feature matrix $Z$ can be used with standard supervised or
semi-supervised learning methods.
In this work, we focus on classification and evaluate performance across
multiple datasets using standard metrics (Section~4). The distance matrix in step 6 is computed using the chordal distance\cite{edelman1998geometry}.
\end{remark}

\section{Experiments}

In this section, we evaluate the proposed HSSE framework through a series
of classification experiments on benchmark single-cell RNA-seq datasets.
The experiments are designed to assess whether the hierarchical sheaf
spectral features improve discriminative performance compared with
existing multiscale geometric representations and commonly used
dimensionality reduction methods.
All methods are evaluated using the same downstream classifier, namely
gradient boosting decision trees (GBDT), with identical hyperparameter
settings across all datasets to ensure a controlled and fair comparison.
\subsection{Datasets and Experimental Setup}

We evaluate the proposed Hierarchical Sheaf Spectral Embedding (HSSE)
framework on twelve benchmark single-cell RNA sequencing datasets that
were curated and analyzed in Feng \emph{et al.}~\cite{feng2024multiscale}.
These datasets are widely used in the literature and cover a broad range
of biological settings, including both human and mouse data, varying
numbers of cells, genes, and annotated cell types.

The datasets span sample sizes from a few hundred to several thousand
cells, gene dimensions on the order of $10^4$--$10^5$, and between 5 and
14 annotated cell types.
A summary of the datasets, including GEO accession IDs, references,
numbers of samples, genes, and cell types, is provided in
Table~\ref{tab:datasets}.
More detailed information on dataset characteristics and preprocessing
can be found in~\cite{feng2024multiscale} and the original references.

\begin{table}[!htbp]
\centering
\caption{Summary of benchmark single-cell RNA-seq datasets used in the experiments. 
Each dataset includes its GEO accession number, reference, number of annotated cell types, 
and the numbers of samples (cells) and genes.}
\label{tab:datasets}
\begin{tabular}{lcccc}
\hline
\textbf{GEO accession} & \textbf{Reference} & \textbf{Cell types} & \textbf{Samples} & \textbf{Genes} \\
\hline
GSE45719 & Deng~\cite{deng2014single} & 8 & 300 & 22\,431 \\
GSE67835 & Dramanis~\cite{darmanis2015survey} & 8 & 420 & 22\,084 \\
GSE75748cell & Chu~\cite{chu2016single} & 7 & 1\,018 & 19\,189 \\
GSE75748time & Chu~\cite{chu2016single} & 6 & 758 & 19\,189 \\
GSE82187 & Gokce~\cite{gokce2016single} & 10 & 705 & 18\,840 \\
GSE84133human1 & Baron~\cite{baron2016single} & 14 & 1\,895 & 20\,125 \\
GSE84133human2 & Baron~\cite{baron2016single} & 14 & 1\,702 & 20\,125 \\
GSE84133human3 & Baron~\cite{baron2016single} & 14 & 3\,605 & 20\,125 \\
GSE84133human4 & Baron~\cite{baron2016single} & 14 & 1\,275 & 20\,125 \\
GSE84133mouse1 & Baron~\cite{baron2016single} & 13 & 822 & 14\,878 \\
GSE84133mouse2 & Baron~\cite{baron2016single} & 13 & 1\,064 & 14\,878 \\
GSE94820 & Villani~\cite{villani2017single} & 5 & 1\,140 & 26\,593 \\
\hline
\end{tabular}
\end{table}

To ensure a fair and controlled comparison, we adopt the same dataset
selection, normalization schemes, and evaluation protocol as
in~\cite{feng2024multiscale}.
Gene expression matrices are normalized using the original normalization
procedures reported for each dataset.
Cell-type labels are used exclusively for evaluation and are not involved
in the construction of HSSE features or any intermediate representations.

For all datasets, UMAP is used to construct the initial low-dimensional
representations from the normalized gene expression matrices.
As a standard preprocessing step for UMAP, principal component analysis
(PCA) is applied prior to UMAP to reduce noise and computational cost.
Following~\cite{wang2025multiscale}, the target UMAP dimension is chosen
according to dataset size:
datasets with fewer than 400 cells are reduced to 15 dimensions,
datasets with 400 to 1200 cells are reduced to 20 dimensions,
and datasets with more than 1200 cells are reduced to 50 dimensions.
The same dimensionality settings are applied across all methods for
consistency.

Based on these low-dimensional representations, HSSE constructs a family
of scale-dependent representations for each dataset.
Specifically, five representation scales are selected using the
power-based sampling strategy described
in~\cite{wang2025multiscale}.
Neighborhood sizes are fixed across all datasets and scales, with
\[
\mathcal{K} = \{5, 10, 15, 20, 30, 40, 50, 60, 70, 80\}.
\]

We compare HSSE against a collection of baseline representation methods
reported in Feng \emph{et al.}~\cite{feng2024multiscale}, including PCA,
UMAP, t-SNE, and nonnegative matrix factorization (NMF).
For a fair comparison, each baseline method produces embeddings with the
same target dimensionality as the initial low-dimensional representations
used in HSSE.
All baseline embeddings are evaluated using the same downstream classifier
and training protocol.

Following~\cite{feng2024multiscale}, gradient boosting decision trees
(GBDT) are employed as the downstream classifier for all feature
representations.
The GBDT hyperparameters are fixed across all datasets and experiments:
number of estimators $=2000$, maximum tree depth $=7$, minimum samples per
split $=5$, learning rate $=0.002$, subsample ratio $=0.8$, and the number
of features considered at each split set to the square root of the input
feature dimension.
No dataset-specific tuning is performed.

Performance is evaluated using five-fold cross-validation.
To reduce the effect of randomness, all experiments are repeated with
different random seeds, and reported results correspond to averaged
performance across runs.
We report standard multi-class classification metrics, including
Macro-Recall, Macro-F1 score, and macro-AUC, which are widely used for
imbalanced multi-class problems and compute performance on a per-class
basis before averaging
\cite{lewis1991evaluating, sokolova2009systematic, hand2001simple}.
A summary of parameter configurations for all datasets is provided in Table S1 (Supporting Information, Section S1), and formal definitions of the evaluation metrics are given in Section S3 of the Supporting Information.

\subsection{Quantitative Results and Comparison}

We quantitatively evaluate the proposed Hierarchical Sheaf Spectral
Embedding (HSSE) by comparing it with the multiscale differential
geometry (MDG) framework of Feng \emph{et al.}~\cite{feng2024multiscale}
and several widely used baseline representations, including PCA, UMAP,
t-SNE, and nonnegative matrix factorization (NMF).
All methods are evaluated under the same downstream classification
setting described in Section~4.1.

Tables~\ref{tab:macro_f1}, \ref{tab:macro_recall}, and
\ref{tab:macro_auc} report the classification performance measured by
Macro-F1, Macro-Recall (balanced accuracy), and Macro-AUC (one-vs-rest),
respectively.
For each metric, results are shown for all twelve datasets, followed by
the average performance across datasets.
\begin{table}[!htbp]
\centering
\caption{Macro-F1 comparison across 12 benchmark single-cell RNA-seq datasets.
HSSE denotes the proposed Hierarchical Sheaf Spectral Embedding.
All baseline results (MDG, PCA, UMAP, NMF, and tSNE) are     taken from Feng \emph{et al.}~\cite{feng2024multiscale}.
For each dataset, the best-performing method is highlighted in bold.}
\label{tab:macro_f1}
\begin{tabular}{lcccccc}
\hline
Dataset & HSSE & MDG  \hspace{-0.5em}\textsuperscript{\cite{feng2024multiscale}} & PCA  \hspace{-0.5em}\textsuperscript{\cite{feng2024multiscale}}  & UMAP  \hspace{-0.5em}\textsuperscript{\cite{feng2024multiscale}}  & NMF  \hspace{-0.5em}\textsuperscript{\cite{feng2024multiscale}}  & tSNE  \hspace{-0.5em}\textsuperscript{\cite{feng2024multiscale}}  \\
\hline
GSE45719 & 0.935 & \textbf{0.959}   & 0.770   & 0.842   & 0.739   & 0.526   \\
GSE67835 & \textbf{0.949} & 0.926   & 0.879   & 0.929   & 0.799   & 0.707   \\
GSE75748cell & \textbf{0.990} & 0.969   & 0.885   & 0.936   & 0.846   & 0.814   \\
GSE75748time & \textbf{0.941} & 0.929   & 0.901   & 0.871   & 0.872   & 0.708   \\
GSE82187 & \textbf{0.986} & 0.950   & 0.860   & 0.968   & 0.852   & 0.815   \\
GSE84133human1 & \textbf{0.961} & 0.949   & 0.636   & 0.946   & 0.622   & 0.895   \\
GSE84133human2 & \textbf{0.967} & 0.950   & 0.798   & 0.952   & 0.710   & 0.931   \\
GSE84133human3 & 0.900 & \textbf{0.934}   & 0.561   & 0.927   & 0.538   & 0.924   \\
GSE84133human4 & 0.981 & \textbf{0.986}   & 0.880   & 0.984   & 0.888   & 0.926   \\
GSE84133mouse1 & \textbf{0.979} & 0.955   & 0.834   & 0.954   & 0.722   & 0.794   \\
GSE84133mouse2 & \textbf{0.969} & 0.936   & 0.682   & 0.948   & 0.710   & 0.847   \\
GSE94820 & \textbf{0.935} & 0.925   & 0.896   & 0.926   & 0.893   & 0.855   \\
\hline
Average & \textbf{0.958} & 0.947  & 0.799  & 0.932  & 0.766  & 0.812  \\
\hline
\end{tabular}
\end{table}

 \begin{table}[!htbp]
\centering
\caption{Macro-Recall comparison across 12 benchmark single-cell RNA-seq datasets.
All baseline results (MDG, PCA, UMAP, NMF, and tSNE) are taken from   Feng \emph{et al.}~\cite{feng2024multiscale}.
Boldface indicates the highest Macro-Recall achieved for each dataset.}
\label{tab:macro_recall}
\begin{tabular}{lcccccc}
\hline
Dataset & HSSE & MDG   \hspace{-0.5em}\textsuperscript{\cite{feng2024multiscale}}  & PCA   \hspace{-0.5em}\textsuperscript{\cite{feng2024multiscale}} & UMAP  \hspace{-0.5em}\textsuperscript{\cite{feng2024multiscale}}  & NMF  \hspace{-0.5em}\textsuperscript{\cite{feng2024multiscale}}  & tSNE  \hspace{-0.5em}\textsuperscript{\cite{feng2024multiscale}}  \\
\hline
GSE45719 & 0.937 & \textbf{0.962}   & 0.780   & 0.848   & 0.747   & 0.532   \\
GSE67835 & \textbf{0.940} & 0.920   & 0.876   & 0.921   & 0.800   & 0.699   \\
GSE75748cell & \textbf{0.990} & 0.968   & 0.887   & 0.934   & 0.848   & 0.811   \\
GSE75748time & \textbf{0.941} & 0.930   & 0.901   & 0.870   & 0.873   & 0.708   \\
GSE82187 & \textbf{0.985} & 0.955   & 0.863   & 0.966   & 0.857   & 0.807   \\
GSE84133human1 & \textbf{0.955} & 0.945   & 0.644   & 0.938   & 0.626   & 0.883   \\
GSE84133human2 & \textbf{0.959} & 0.950   & 0.805   & 0.951   & 0.715   & 0.916   \\
GSE84133human3 & 0.890 & \textbf{0.935}   & 0.575   & 0.918   & 0.550   & 0.909   \\
GSE84133human4 & 0.986 & \textbf{0.988}   & 0.884   & 0.984   & 0.895   & 0.924   \\
GSE84133mouse1 & \textbf{0.975} & 0.952   & 0.826   & 0.949   & 0.719   & 0.783   \\
GSE84133mouse2 & \textbf{0.967} & 0.935   & 0.684   & 0.940   & 0.716   & 0.839   \\
GSE94820 & \textbf{0.930} & 0.925   & 0.898   & 0.925   & 0.895   & 0.857   \\
\hline
Average & \textbf{0.954} & 0.947  & 0.802  & 0.929  & 0.770  & 0.806  \\
\hline
\end{tabular}
\end{table}

\begin{table}[!htbp]
\centering
\caption{Macro-AUC  comparison across 12 benchmark single-cell RNA-seq datasets.
All baseline results (MDG, PCA, UMAP, NMF, and tSNE) are taken  from Feng \emph{et al.}~\cite{feng2024multiscale}.
For each dataset, the best result is shown in bold.}
\label{tab:macro_auc}
\begin{tabular}{lcccccc}
\hline
Dataset & HSSE & MDG   \hspace{-0.5em}\textsuperscript{\cite{feng2024multiscale}}  & PCA   \hspace{-0.5em}\textsuperscript{\cite{feng2024multiscale}}  & UMAP  \hspace{-0.5em}\textsuperscript{\cite{feng2024multiscale}}  & NMF  \hspace{-0.5em}\textsuperscript{\cite{feng2024multiscale}}  & tSNE  \hspace{-0.5em}\textsuperscript{\cite{feng2024multiscale}}  \\
\hline
GSE45719 & 0.991 & \textbf{0.994}   & 0.950   & 0.965   & 0.944   & 0.831   \\
GSE67835 & 0.986 & 0.986   & 0.984   & \textbf{0.992}   & 0.975   & 0.920   \\
GSE75748cell & \textbf{0.998} & 0.996   & 0.983   & 0.992   & 0.974   & 0.965   \\
GSE75748time & 0.989 & \textbf{0.990}   & 0.979   & 0.971   & 0.980   & 0.927   \\
GSE82187 & \textbf{1.000} & 0.995   & 0.992   & 0.993   & 0.988   & 0.961   \\
GSE84133human1 & \textbf{0.989} & 0.987   & 0.951   & 0.979   & 0.944   & 0.980   \\
GSE84133human2 & \textbf{0.994} & 0.991   & 0.984   & 0.989   & 0.977   & 0.988   \\
GSE84133human3 & 0.980 & 0.989   & 0.951   & 0.989   & 0.940   & \textbf{0.990}   \\
GSE84133human4 & \textbf{0.998} & 0.996   & 0.991   & \textbf{0.998}   & 0.993   & 0.990   \\
GSE84133mouse1 & \textbf{0.999} & 0.987   & 0.972   & 0.995   & 0.954   & 0.961   \\
GSE84133mouse2 & \textbf{0.993} & 0.989   & 0.967   & 0.992   & 0.967   & 0.979   \\
GSE94820 & \textbf{0.990} & 0.986   & 0.982   & 0.992   & 0.977   & 0.976   \\
\hline
Average & \textbf{0.992} & 0.991  & 0.974  & 0.987 & 0.968  & 0.956  \\
\hline
\end{tabular}
\end{table}

\noindent
\textbf{Comparison with MDG.}
Tables~2--4 summarize the absolute classification performance of HSSE,
MDG, and baseline methods across all datasets.
To further highlight the relative behavior of HSSE compared with MDG,
Figure~\ref{fig:hsse_mdg_dumbbell} visualizes the dataset-wise performance
differences using paired dumbbell plots for Macro-F1, Macro-Recall, and
Macro-AUC.

\begin{figure*}[!ht]
    \centering
    \hspace*{-0.5cm}
    \includegraphics[width=1.05\textwidth]{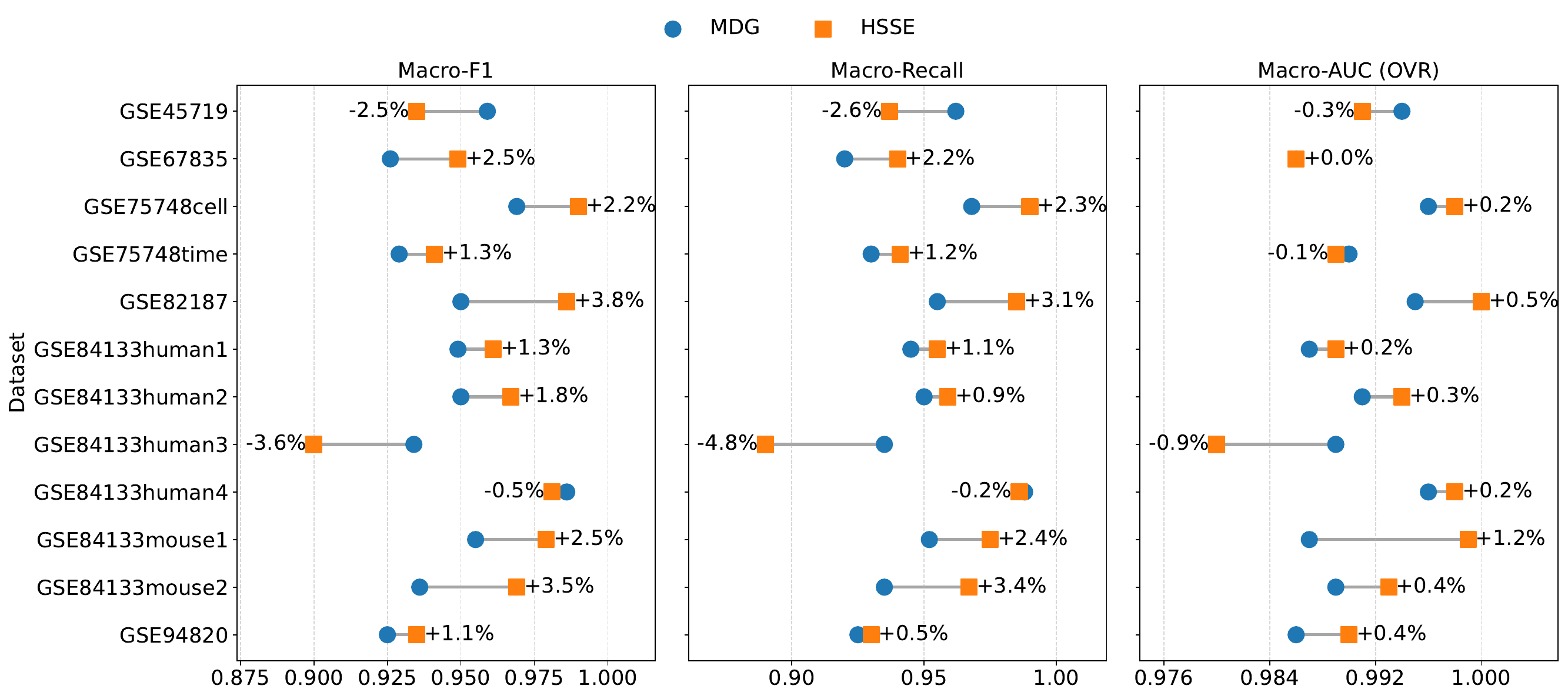}
    \caption{
    Dataset-wise performance differences between HSSE and MDG across three evaluation metrics.
    Each panel shows paired results for HSSE and MDG on the same dataset using a dumbbell plot:
    (left) Macro-F1, (middle) Macro-Recall, and (right) Macro-AUC (OVR).
    Horizontal line segments connect MDG (circle) and HSSE (square), and the percentage annotation
    indicates the relative change of HSSE compared to MDG.
    Positive values denote improvements achieved by HSSE, while negative values indicate datasets
    where MDG performs slightly better.
    }
    \label{fig:hsse_mdg_dumbbell}
\end{figure*}

As shown in Figure~\ref{fig:hsse_mdg_dumbbell}, HSSE achieves consistent
performance gains over MDG on a majority of datasets, particularly for
Macro-F1 and Macro-Recall.
Notable improvements are observed on GSE67835, GSE75748cell, GSE82187,
and both mouse pancreatic datasets (GSE84133mouse1 and
GSE84133mouse2), where HSSE improves balanced classification accuracy by
approximately 2\%--4\%.
These gains suggest that incorporating cell-centered persistent sheaf
spectral information provides complementary discriminative power beyond
the multiscale geometric embeddings employed by MDG.

For Macro-AUC, HSSE remains highly competitive with MDG across all
datasets and exhibits small but consistent improvements on most tasks,
indicating robust class separability under the one-vs-rest setting.
Only a small number of datasets, most notably GSE84133human3, show a
slight performance advantage for MDG.
This dataset is the largest and most heterogeneous human pancreatic
dataset considered, and the reduced performance gap suggests that the
benefit of local sheaf-based constructions may diminish when global
structure dominates.
This behavior is discussed further in Section~5.

 \noindent
\textbf{Comparison with Classical Methods.}
Compared with classical dimensionality reduction techniques, including
PCA, UMAP, NMF, and t-SNE, HSSE demonstrates clear and consistent
advantages across all three evaluation metrics.
As shown in Tables~2--4 and Fig.~\ref{fig:hsse_classical_gain}, PCA and NMF
generally underperform HSSE by a substantial margin in both Macro-F1 and
Macro-Recall, reflecting their limited ability to capture nonlinear and
multiscale structures inherent in single-cell RNA-seq data.

UMAP and t-SNE achieve strong performance on several datasets, and in a
few cases approach or exceed MDG.
However, their performance exhibits noticeably higher variability across
datasets, particularly in terms of Macro-F1 and Macro-Recall.
In contrast, HSSE maintains more stable and consistently positive gains
over all classical methods across datasets with varying sizes, gene
dimensions, and degrees of cell-type heterogeneity.
This stability is reflected in HSSE achieving the highest average
Macro-F1 and Macro-Recall scores among all compared methods.

For Macro-AUC, all methods achieve relatively high values, indicating
strong overall class separability across datasets.
Nevertheless, HSSE consistently ranks among the top-performing
approaches and attains the highest average Macro-AUC, as illustrated in
Fig.~\ref{fig:hsse_classical_gain}, demonstrating that HSSE not only
improves balanced classification performance but also preserves robust
discriminative power under the one-vs-rest evaluation setting.

\begin{figure*}[!ht]
    \centering
    \includegraphics[width=\textwidth]{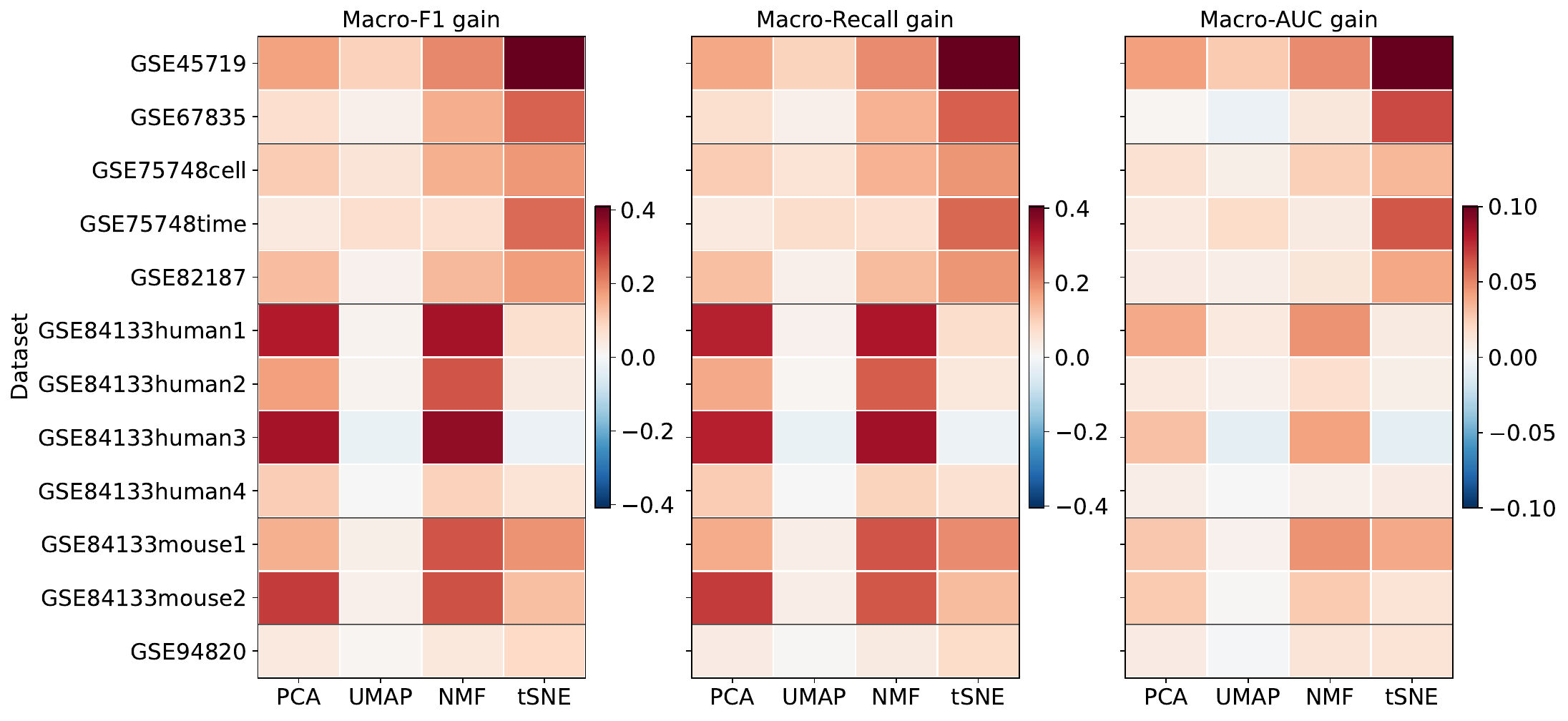}
    \caption{
    Performance gains of HSSE over classical dimensionality reduction
    methods (PCA, UMAP, NMF, and t-SNE) across twelve benchmark single-cell
    RNA-seq datasets.
    Each panel reports the gain of HSSE in terms of
    (left) Macro-F1,
    (middle) Macro-Recall, and
    (right) Macro-AUC, respectively.
    Positive values indicate consistent performance improvements achieved
    by HSSE over the corresponding baseline method.
    }
    \label{fig:hsse_classical_gain}
\end{figure*}

\noindent
\textbf{Overall Performance.} An overall analysis across datasets of varying sizes and complexities
demonstrates that HSSE provides a robust and competitive representation
for single-cell classification.
Across the twelve benchmark datasets, HSSE consistently achieves strong
performance in terms of Macro-F1, Macro-Recall, and Macro-AUC, indicating
balanced classification accuracy as well as reliable class separability.

HSSE is particularly effective on datasets of moderate scale, roughly
ranging from 400 to 2000 cells, where local neighborhood structure can be
reliably estimated while still benefiting from multiscale geometric
information.
Datasets such as GSE67835, GSE75748cell, GSE82187, and the mouse pancreatic
datasets exemplify this regime, with HSSE achieving consistently high
scores across all three evaluation metrics.

On very small datasets (e.g., GSE45719) and on the largest dataset
considered (GSE84133human3), the performance gap between HSSE and MDG
narrows or reverses slightly.
This behavior suggests that the effectiveness of persistent sheaf
spectral features depends on a balance between neighborhood resolution
and sample size.
When the number of cells is limited, local spectral statistics may become
less stable, whereas for extremely large and heterogeneous datasets,
global geometric structure may dominate the learned representations.

Overall, the results reported in Tables~2--4 indicate that by integrating
multiscale geometry with interpretable persistent sheaf spectral
information, HSSE achieves consistent improvements over MDG and
classical dimensionality reduction methods across a diverse collection
of benchmark single-cell RNA-seq datasets.
These findings highlight the robustness and general applicability of the
proposed framework.

Additional evaluation metrics, including Accuracy, Weighted F1, and MCC, are reported for all twelve datasets in Section S4 of the Supporting Information.  

\subsection{Ablation and Sensitivity Analysis}
\label{sec:43}

To better understand the empirical behavior of the proposed HSSE
framework, we conduct a focused ablation and sensitivity analysis with
respect to two key hyperparameters: the number of neighborhood sizes
used in multi-scale aggregation ($\mathcal{K}$), and the scale configuration
($S$) from which these neighborhood sizes are selected.
These two parameters jointly determine how local and nonlocal
neighborhood information is incorporated, and therefore directly
affect the representation quality.

Throughout this section, all experiments follow the same experimental
protocol described in Section~4.1.
When analyzing one parameter, all remaining settings are kept fixed to
isolate its individual effect.
Unless otherwise stated, the same four representative datasets
(GSE67835, GSE82187, GSE84133mouse1, and GSE94820) are used for all
experiments in this section. The full numerical results are provided in  Tables S3--S10 (Supporting Information, Section S5).

\subsubsection{Effect of the Number of Neighborhood Sizes ($\mathcal{K}$)}
\label{sec:43_k}

We first investigate how the classification performance varies as the
number of neighborhood sizes used in the multi-scale aggregation
increases.
In this experiment, the scale configuration is fixed for all datasets,
with the scale pool given by
\[
S = \{5, 14, 25, 37, 50\}.
\]
Starting from this fixed scale pool, we progressively increase the
number of neighborhood sizes used by the method, from $k=1$ to $k=10$,
while keeping all other hyperparameters unchanged.
This design allows us to examine the contribution of incorporating
additional neighborhood sizes without altering the underlying scale
configuration.

\begin{figure}[t]
    \centering
    \includegraphics[width=\textwidth]{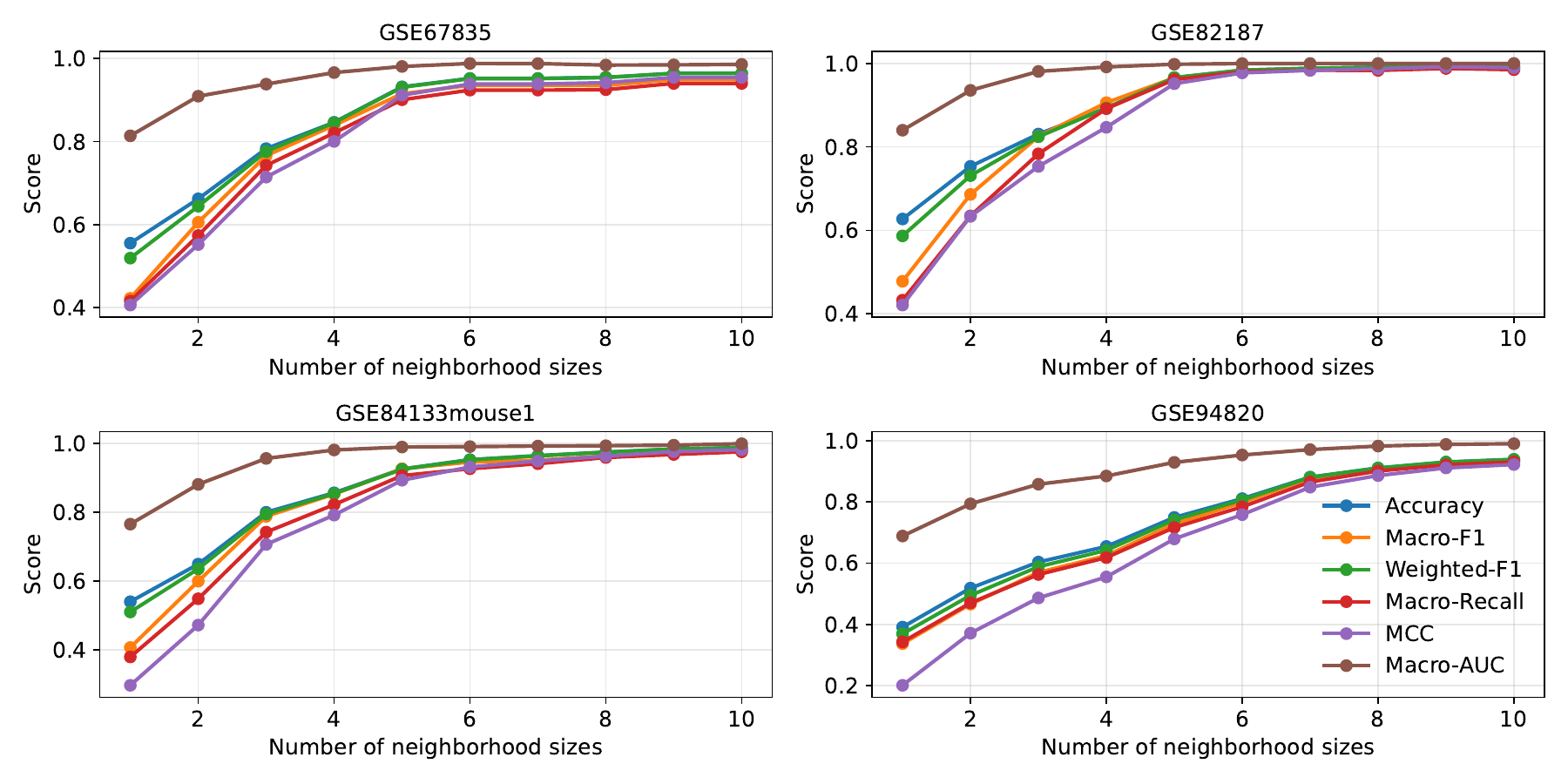}
    \caption{
    Sensitivity analysis with respect to the number of neighborhood sizes
    under a fixed scale configuration
    $S = \{5, 14, 25, 37, 50\}$.
    Results are shown for four representative datasets:
    GSE67835, GSE82187, GSE84133mouse1, and GSE94820.
    The horizontal axis denotes the number of neighborhood sizes used
    ($|\mathcal{K}|=1$ to $10$).
    Six evaluation metrics are reported: Accuracy, Macro F1, Weighted F1,
    Macro Recall, MCC, and Macro AUC.
    }
    \label{fig:k_sweep_six_metrics}
\end{figure}

As shown in Figure~\ref{fig:k_sweep_six_metrics}, a consistent trend is
observed across all datasets.
When only a small number of neighborhood sizes is used, the performance
is relatively limited, suggesting that very restricted neighborhood
context is insufficient to fully capture the underlying data structure.
As $k$ increases, all metrics improve substantially in the early range,
indicating that aggregating information from multiple neighborhood
sizes provides complementary structural cues.

For GSE67835, most metrics increase rapidly as $|\mathcal{K}|$ grows from 1 to
approximately 5.
Beyond this point, the curves become noticeably flatter, and further
increasing $|\mathcal{K}|$ leads to only modest changes.
Macro AUC reaches a high level early and remains stable, while Macro F1,
Macro Recall, and MCC continue to improve slightly before stabilizing.

A similar pattern is observed for GSE82187.
Performance improves markedly for small to moderate values of $k$ and
then remains close to a plateau.
Minor non-monotonic variations appear for some metrics at larger $k$,
but their magnitudes are limited, indicating that the method is not
highly sensitive to the exact choice of $k$ once a sufficient number of
neighborhood sizes is included.

For GSE84133mouse1, the early improvement is again pronounced.
While most metrics stabilize after moderate values of $k$, MCC exhibits
a more gradual increase and continues to change slightly even when other
metrics have largely saturated.
This behavior is consistent with MCC being more sensitive to residual
class-level confusion.

In contrast, GSE94820 shows a more gradual and sustained improvement
over the full range of $k$ considered.
The absence of a sharp plateau within this range suggests that this
dataset benefits more consistently from incorporating additional
neighborhood sizes.

Overall, these results indicate that incorporating multiple neighborhood
sizes is important for achieving strong performance, with the largest
gains occurring when increasing $k$ from very small to moderate values.
Beyond this range, the performance often becomes less sensitive to
further increases in $k$, although the degree of saturation depends on
the dataset.

\subsubsection{Effect of the Scale Configuration Size ($S$)}
\label{sec:43_s}

We next examine the influence of the scale configuration by varying the
number of scales included in the scale pool $S$.
In this analysis, all experimental settings are kept fixed except for the
size of $S$, which is progressively increased from one to five by adding
scales from a predefined pool.
Throughout this experiment, the set of neighborhood sizes used in
aggregation is fixed to
\[
\mathcal{K} = \{5, 10, 15, 20, 30, 40, 50, 60, 70, 80\}.
\]
Therefore, this study isolates the effect of the scale configuration
($S$) rather than changes in the neighborhood-size selection.
\begin{figure}[t]
  \centering
  \includegraphics[width=\linewidth]{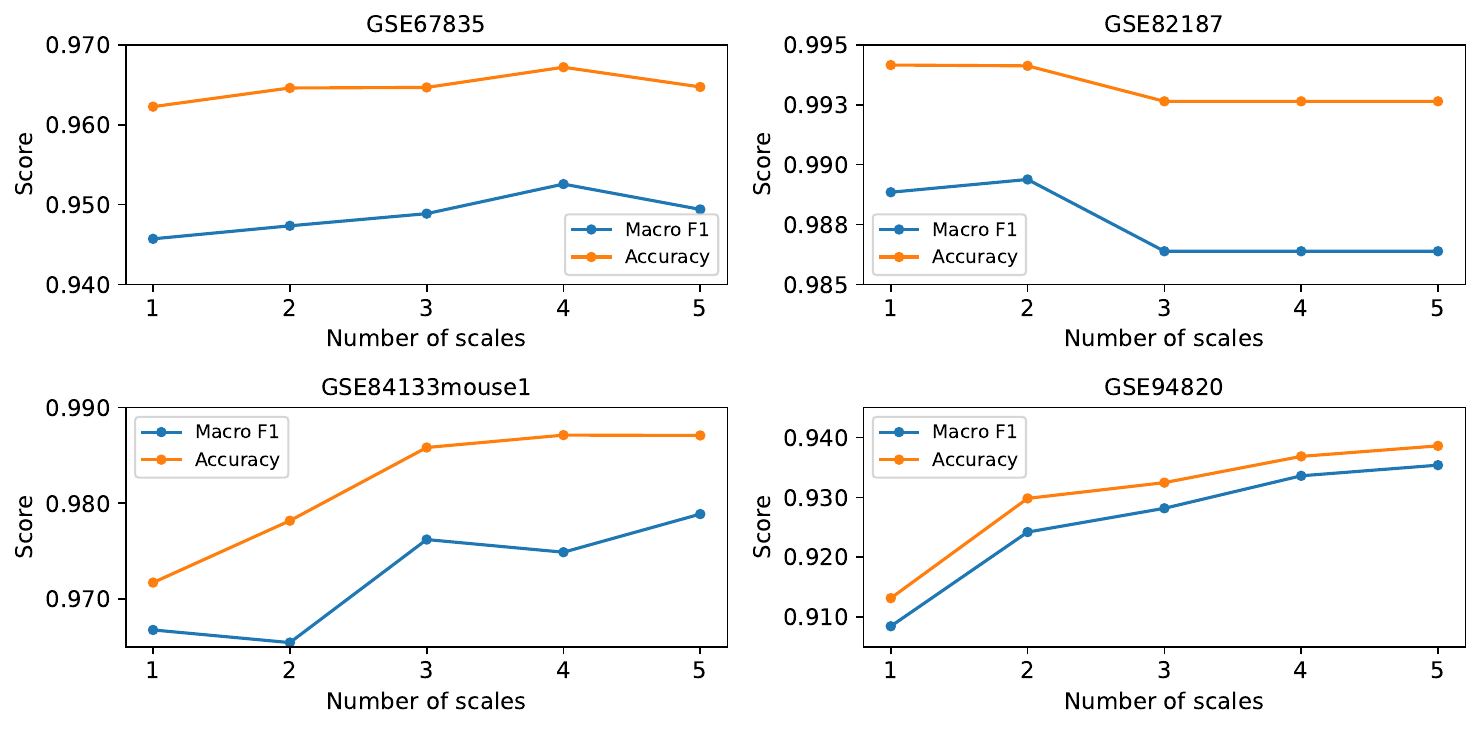}
  \caption{
  Effect of the scale configuration size ($|S|$) on classification
  performance under a fixed neighborhood-size set
  $\mathcal{K} = \{5, 10, 15, 20, 30, 40, 50, 60, 70, 80\}$.
  Macro F1 and Accuracy are reported as $|S|$ increases from 1 to 5 on four
  representative datasets: GSE67835, GSE82187, GSE84133mouse1, and GSE94820.
  }
  \label{fig:scale_two_metrics}
\end{figure}

Figure~\ref{fig:scale_two_metrics} summarizes the performance trends as
the number of scales in $S$ increases.
Across all datasets, using multiple scales consistently outperforms the
single-scale setting, indicating that aggregating neighborhood information
at different resolutions captures complementary structural patterns that are
not available from a single scale alone.

At the same time, the performance does not increase monotonically with
the size of $S$.
Instead, improvements are most apparent when moving from one to a small
number of scales, after which the gains gradually diminish and the curves
either saturate or exhibit mild fluctuations.
This behavior is consistent across datasets with varying levels of
difficulty.

For example, on GSE67835 and GSE82187, the performance gains saturate after
incorporating only a few scales.
On GSE84133mouse1 and GSE94820, adding additional scales leads to further
improvements before reaching a plateau.
Importantly, these variations are small in magnitude, suggesting that the
proposed framework is not overly sensitive to the exact choice of the scale
configuration size within the range considered.

Based on this analysis, five scales are used as the default scale
configuration in the experiments reported in this paper.
This choice provides consistently strong performance across datasets
while avoiding unnecessary complexity and excessive fine-tuning.

\section{Discussion}

This work investigates hierarchical sheaf spectral embeddings for
multi-scale representation learning, with particular emphasis on the
roles of neighborhood size, scale configuration, and topological
information across different homological dimensions.
The experimental results demonstrate that the proposed framework
achieves strong and stable performance across diverse datasets, while
exhibiting robust behavior with respect to key design choices.

A central observation from the experimental analysis is the importance
of incorporating multi-scale neighborhood information.
As shown in the sensitivity studies, increasing the number of
neighborhood sizes leads to substantial performance improvements
compared to using only a single neighborhood size.
The most pronounced gains occur when moving from very limited local
context to a moderate number of neighborhood sizes, after which the
performance tends to stabilize.
This pattern is consistently observed across datasets with different
sizes and class structures, indicating that the proposed framework
effectively integrates complementary information across scales while
remaining relatively insensitive to precise parameter choices.

A similar trend is observed when varying the scale configuration size
$|S|$ under a fixed set of neighborhood sizes.
Using multiple scales consistently outperforms the single-scale setting,
highlighting the benefit of aggregating information across different
resolutions.
At the same time, increasing the number of scales beyond a moderate
level yields diminishing returns, with performance curves either
saturating or exhibiting only minor fluctuations.
These results suggest that the proposed framework is robust to the
choice of scale configuration within a reasonable range, and that strong
performance can be achieved without extensive tuning of scale-related
parameters.

Beyond parameter sensitivity, it is also instructive to examine the
contribution of topological information at different homological
dimensions.
To this end, additional experiments reported in Section S2 of the Supporting
Information compare the performance obtained using only
$0$-dimensional eigenvalues with that obtained by combining
$0$- and $1$-dimensional information (see Fig.~S1).
Across multiple datasets, $0$-dimensional features alone are found to
capture a substantial portion of the discriminative structure.
In several cases, the performance achieved using only
$0$-dimensional information is close to that obtained when both
$0$- and $1$-dimensional eigenvalues are included.

These observations indicate that connectivity-based topological
information plays a dominant role for many of the datasets considered in
this study.
In particular, $0$-dimensional features effectively encode local
neighborhood organization and cluster-level structure, which are often
sufficient to distinguish classes in high-dimensional data.
At the same time, incorporating $1$-dimensional information can provide
additional gains in certain cases, reflecting the presence of more
complex relationships such as loops or higher-order interactions.
The magnitude of these gains, however, is dataset dependent and not
uniformly significant.

From a practical perspective, the choice between using only
$0$-dimensional information and incorporating higher-dimensional
topological features involves a trade-off between computational cost and
potential performance improvements.
While $0$-dimensional features provide an efficient and stable baseline,
including $1$-dimensional information offers a complementary enhancement
when computational resources permit.
Rather than advocating a fixed choice, the proposed framework supports
flexible integration of topological information at different dimensions,
allowing practitioners to balance representational richness against
computational considerations.

Overall, these results show that HSSE achieves robust performance by
aggregating spectral information across multiple scales and neighborhood
configurations, rather than relying on a single topological feature or a
carefully tuned parameter setting.
This aggregation strategy reduces sensitivity to individual design
choices and yields consistent classification performance across diverse
single cell datasets, without requiring extensive parameter optimization.
 
\section{Conclusion}

In this work, we proposed a hierarchical sheaf spectral embedding (HSSE)
framework for single-cell representative learning.
The framework integrates multi-scale low-dimensional embeddings with
persistent sheaf Laplacian analysis to extract cell-level features that
encode local relational structure across scales.
By aggregating spectral information from multiple neighborhood sizes,
embedding scales, and filtration values, HSSE produces representations
that can be directly used in downstream classification tasks.
Comprehensive experiments on twelve benchmark single-cell RNA
sequencing datasets demonstrate that the proposed framework achieves
competitive and, in many cases, improved classification performance
compared with existing manifold learning and topology-based methods.
Sensitivity analyses with respect to scale selection and neighborhood
size indicate that performance improvements arise from the hierarchical
aggregation strategy rather than reliance on a specific parameter
choice.
The results suggest that persistent sheaf Laplacians provide a useful
tool for modeling local structural variation in single-cell data.
While the current study focuses on classification tasks, the proposed
framework can be naturally extended to other downstream analyses.
Future work will explore alternative sheaf constructions, more
efficient implementations, and applications to additional types of
biological data, such as single-cell evolution and spatial transcriptomics data.

\section*{Supporting Information}

Additional materials, including experimental setup, evaluation metrics, and complete results for the sensitivity analysis, are provided in the Supporting Information.

\section*{Data and Code Availability}

All source code is available at \url{https://github.com/XiangXiangJY/HSSE.git}, and the single-cell RNA-seq dataset preprocessing scripts are adapted from  
\url{https://github.com/hozumiyu/TopologicalNMF-scRNAseq}.

	    \section*{Acknowledgements}
This work was supported in part by NIH grant R35GM148196, National Science Foundation grant  DMS2052983, Michigan State University Research Foundation, and Bristol-Myers Squibb 65109.

\bibliographystyle{plain}
\bibliography{refs}
   
\end{document}